%% file: ijcnn2020.tex
\documentclass[conference]{IEEEtran}
\IEEEoverridecommandlockouts
\usepackage{cite}
\usepackage{amsmath,amssymb,amsfonts}
\usepackage{algorithmic}
\usepackage{float}
\usepackage{graphicx,subfigure}
\usepackage{textcomp}
\usepackage{xcolor}
\usepackage{tikz}
\usetikzlibrary{arrows.meta}

\usepackage{grffile}
\usepackage{amsthm}
\theoremstyle{plain}
\newtheorem{proposition}{Proposition}

\usepackage{xcolor}

\usepackage{mathrsfs}

\def\F{\mathscr{F}}

\def\A{\mathscr{A}}
\DeclareMathOperator{\vect}{vec}

\def\BibTeX{{\rm B\kern-.05em{\sc i\kern-.025em b}\kern-.08em
T\kern-.1667em\lower.7ex\hbox{E}\kern-.125emX}}

\def\R{{\bf R}}

\def\x{\xi}
\def\W{M}
\def\w{m}
\def\tvar{\tau}

\newcommand\blfootnote[1]{%
  \begingroup
  \renewcommand\thefootnote{}\footnote{#1}%
  \addtocounter{footnote}{-1}%
  \endgroup
}

\begin{document}

\title{Developing Constrained Neural Units Over Time
\thanks{}
}

\author{\IEEEauthorblockN{Alessandro Betti, Marco Gori, Simone Marullo, Stefano Melacci}
\IEEEauthorblockA{\textit{Department of Information Engineering and Mathematics} \\
\textit{University of Siena}\\
Siena, Italy \\
alessandro.betti2@unisi.it, marco@diism.unisi.it, simone.marullo@student.unisi.it, mela@diism.unisi.it}
}

\maketitle

\begin{abstract}
In this paper we present a foundational study on a constrained method
that defines learning problems with Neural Networks in the
context of the principle of least cognitive action, 
which very much resembles the principle of least action in mechanics.
Starting from a general approach to enforce constraints
into the dynamical laws of learning, this work focuses on an
alternative way of defining Neural Networks, that is different
from the majority of existing approaches.
In particular, the structure of the neural architecture is defined by means of a special
class of constraints that are extended also to the interaction with data, leading to
``architectural'' and ``input-related'' constraints, respectively.
The proposed theory is cast into the time domain, in which data are presented to the network in
an ordered manner, that makes this study an important step toward alternative ways of processing continuous streams of data with Neural Networks.
The connection with the classic Backpropagation-based update rule of the weights of networks is discussed, showing that there are conditions under which our approach degenerates to Backpropagation.
Moreover, the theory is experimentally evaluated on a simple problem that allows us to deeply study several aspects of the theory itself and to 
show the soundness of the model.
\end{abstract}


\section{Introduction}
\label{sec:2nd-order-reduction-proof}
\blfootnote{Accepted for publication at the IEEE International Joint Conference on Neural Networks (IJCNN) 2020 (DOI: TBA).

© 2020 IEEE. Personal use of this material is permitted. Permission from IEEE must be obtained for all other uses, in any current or future media, including reprinting/republishing this material for advertising or promotional purposes, creating new collective works, for resale or redistribution to servers or lists, or reuse of any copyrighted component of this work in other works.}
When a large amount of supervised data is available, (Deep) Neural Networks have shown to yield impressive results in several real-world tasks \cite{krizhevsky2012imagenet,graves2013speech,xu2015show}.
The classic formulation of supervised learning in Neural Networks consists in optimizing the values of the weights attached to a pre-designed neural architecture in order to fit the given training data under some regularity conditions or, in any case, to control the learning process to avoid overfitting and gain generalization skills \cite{gori2017machine}. Stochastic gradient descent \cite{bottou2012stochastic} is commonly regarded as the \textit{de facto} schema for the optimization of the network weights. According to this approach, at each time step only one sample (or a mini-batch of samples) is considered, and the training data are shuffled before the beginning of any training epoch, completely ignoring any information  eventually available in the data ordering. Whenever the system is designed to learn from data as soon as they become available over time, we move a step toward a real online learning setting. Neural Networks are also exploited in this setting, even if there are several other challenging issues than are not present in the static batch case \cite{parisi2019continual}, and stochastic gradient descent can still be applied, updating the network weights after having processed each newly received sample.

In this paper we follow the ideas behind the principle of  least cognitive action which very much resembles the principle of least action in mechanics, and that was investigated in the context of Neural Networks and Computer Vision in \cite{betti2019cognitive}, developing the so-called \textit{Cognitive Action Laws of Learning}. Such learning framework naturally deals with learning problems in which the time component plays a crucial role, so that it is well-suited to approach in a principled way those cases in which data become available over time, that is the setting of this paper. 

Starting from the seminal work of \cite{lecun1988theoretical}, optimization problems on the weights
of Neural Networks can be equivalently formulated by extending the space of the learnable parameters, introducing a constrained optimization problem on the product space of weights and other neuron-related variables, as also investigated in more recent works \cite{carreira2014distributed,taylor2016training,noia}.
This point of view allows us to describe the  structure of the architecture of the network in terms of constraints involving such extended set of learnable parameters. Moreover, a constraint-based description has implications in the way the network operations can be parallelized and, potentially, in the way the structure of the network is progressively developed, thus offering a very generic perspective on which to build a foundational study.

In this paper we develop this idea by introducing the so-called \textit{architectural constraints} among neurons, and also \textit{input constraints} between each input neuron and the corresponding feature of the considered input example, thus fully breaking down the description of the network in terms of constraints.
We show how these constraints can be applied within the
context of the principle of least cognitive action to model a temporal 
evolution of the variables of the network. This approach allows us to formulate a learning problem that is intrinsically linked, from the beginning,
to the idea of dealing with a temporal trajectory of the weights, as opposed to the classical approaches
where the definition of the learning problem and the choice of the
optimization algorithm consist of two separate and conceptually
orthogonal aspects. 

Our foundational study is developed using the formalism of calculus of variations
with subsidiary conditions. We show how it is possible to formulate the theory using Lagrange multipliers and how those multipliers can be explicitly found by solving a linear system in as many unknowns as the number of neurons in the network. We also discuss how this theory, in a specific regime of the regularization parameters and in certain conditions of dissipation, exactly reproduces the Backpropagation rule for the
computations of the gradients.

There are many reasons to pursue this {\em dynamical} approach to learning \cite{betti2019cognitive,ijcai2019-278}, especially in those
learning problems where the data comes as a (temporally) coherent signal.
Dynamical constraints, such as invariance under motion \cite{ijcai2019-278}, can be valuable in guiding the learning process, and the theory presented in this paper is
developed as a building block to enable the implementation and the analysis of
such temporal constraints for deep architectures (for example, in order to overcome the layer-wise training limitations of \cite{ijcai2019-278}).

This paper is organized as follows. Section \ref{sec:constraints} describes the proposed theoretical study, that is experimentally assessed focusing on a simple and easy-to-understand problem in Section \ref{sec:exp}. Finally, Section \ref{sec:conclusions} concludes the paper with our ideas about future work.

\section{Learning Framework}
\label{sec:constraints}
Our learning framework is rooted on the idea of describing a Neural Network by means of constraints among neural units.
Before going into a formal definition and describing all the details, we illustrate the basic idea using a running example. Then, in Section \ref{sec:least} we will inject the network into the time domain, while the connection with Backpropagation is formalized in Section \ref{sec:bp}.

\indent
\parshape 7
0pc 15pc
0pc 15pc
0pc 15pc
0pc 15pc
0pc 15pc
0pc 15pc
0pc \hsize
\vadjust{\moveright 15pc\vbox to 0pt{\vskip -.8pc
\hsize=0pt\vtop{\includegraphics{./ml-33.mps}}\vss}}
The diagram on the right shows a feed-forward network with $2$ input units (triangles), $2$ hidden and $1$ output neurons (circles). The graph of the architecture makes explicit the way in which the values associated with the
input nodes (vertices $0$ and $1$) are propagated up until the output of the
network (vertex $4$). In particular, if we denote with  $x^i$ the output of
the $i$-th neuron and with $w_{ij}$ the weight associated with the arch
$j \rightarrow i$, then the diagram in the figure implies
$x^2=\sigma(w_{20}x^0+w_{21}x^1)$, $x^3=\sigma(w_{30}x^0+w_{31}x^1)$
and $x^4=\sigma(w_{42}x^2+w_{43}x^3)$, being $\sigma$ the activation function. Therefore, in the
$(w,x)$ space these compositional relations between the nodes variables $x^i$, $i=0,\ldots,4$ can be
regarded as \textit{architectural constraints}, namely $G^2=G^3=G^4=0$, where:
\[G^2=x^2-\sigma(w_{20}x^0+w_{21}x^1), \quad
G^3=x^3-\sigma(w_{30}x^0+w_{31}x^1),\]
\[G^4=x^4-\sigma(w_{42}x^2+w_{43}x^3).\]
Moreover the way in which input signals are provided to the network can be regarded as additional constraints. Let us suppose we want to
compute the value of the network on the input $x^0=e^0$ and $x^1=e^1$, where
$e^0$ and $e^1$ are two scalar values; this two assignments can be interpreted as
two additional \textit{input constraints} $G^0=G^1=0$ where
$$G^0=x^0-e^0,\quad G^1=x^1-e^1.$$

We can now formally describe the models that we consider in this paper, that in Section \ref{sec:least} will be used to formulate the learning problem
in the context of the time domain.
Given a simple digraph ${D=(N,A)}$ of order $\nu$, where $N$ is the set of vertices, while $A$ is the set of edges. Without loss of
generality, we can assume ${N=\{0,1,\dots,\nu-1\}}$ and $A\subseteq N\times N$. A neural network associated with $D$
consists of a set of maps
$i\in N\mapsto x^i\in \R$ and 
$(i,j)\in A\mapsto
w_{ij}\in \R$ together with $\nu$ constraints $G^j(x,W)=0$,
$j=0,1,\dots \nu-1$ where $W$ is the weight matrix (i.e., $(W)_{ij}=w_{ij}$) and $x$ the vector of the outputs of the neurons.
Let ${\cal M}_\nu(\R)$ be the
set of all $\nu\times\nu$ real matrices and 
${\cal M}^{\downarrow}_\nu(\R)$ the set of all $\nu\times\nu$
strictly lower triangular matrices over $\R$. In the rest of the paper we will
always assume that
$W\in{\cal M}^{\downarrow}_\nu(\R)$, i.e.,  a directed acyclic graph.
The relations
$G^j=0$ for $j=0,\dots, \nu-1$
 specify the computational scheme with which the information
diffuses through the network. 
 
In a typical network with $\omega$ inputs, the structure of these
constraints are defined as follows (see also Fig.~\ref{neuron-const-fig}).
For any vector $\x\in\R^\nu$, for any matrix $\W\in{\cal
M}_\nu(\R)$ with entries $\w_{ij}$ and for any given ${\cal C}^1$ map (i.e., a differentiable map whose derivative is continuous)
$e\colon [0,+\infty)\to\R^\omega$ we define the constraint
function on neuron $j$
when the example $e(\tau)$ is presented to the network as
\begin{equation}G^j(\tau,\x,\W):=\begin{cases}\
\x^j-e^j(\tau), & \text{if $0\le j < \omega$};\\
\x^j-\sigma(\w_{jk} \x^k) & \text{if $\omega \le j < \nu$},
\end{cases}\label{neuron-constraints-structure}\end{equation}
where $\sigma\colon \R\to\R$ is of class ${\cal C}^2(\R)$ (i.e., it is differentiable and its derivative is of class ${\cal C}^1(\R)$), and we used the
Einstein notation (i.e., there is a sum on the index $k$), that we will exploit throughout the whole paper to  simplify the notation.\footnote{
Summation is intended over all repeated indices of an expression. 
}
With this choice it is clear that the set of all constraints $G^j(\tau,x,W)=0$, $j=0,\dots,\nu-1$
is the usual computational scheme of a feed-forward network with input $e(\tau)$ whenever we choose
$W\in {\cal M}^{\downarrow}_\nu(\R)$. Throughout the paper we will use
the notation $G^j_\tau$, $G^j_\xi$, $G^j_M$ for the partial derivatives
with respect to  the first, second and third arguments
of $G^j$, respectively.
Notice that the dependence of the constraints on $\tau$ reflects
the fact that the computations of a neural network should be based on
external inputs.

\def\|#1|{\vbox{\hbox{\includegraphics{./const-#1.mps}}}}
\def\<#1>{\hbox{\includegraphics{./const-#1.mps}}}
\begin{figure}[t]
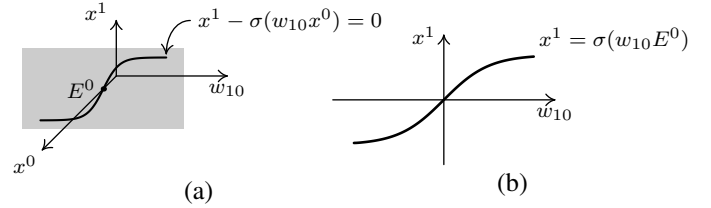

\[{\vcenter{\|5|}\atop\hbox{(a)}}
\mskip -40 mu
{\vcenter{\|3|}\atop\hbox{(b)}}\]
\vskip -1pc
\caption{Visualization of the neural constraints
for the neural network \protect\<4> (two neurons $\{0,1\}$).  Constraint $G^1(x,W)=0$, 
restricted to the plane $x^0=E^0$, is shown in (a). In (b), 
such restriction is represented in the  $w_{10}$--$x^1$ plane.}
\label{neuron-const-fig}
\end{figure}

\subsection{Principle of Least Cognitive Action}
\label{sec:least}
Human cognitive processes do not emerge with a well-defined distinction between training and test set. As time goes by, humans react efficiently to new stimuli, which suggests us to look for alternative foundations of learning by embedding the agent into its own time-driven learning environment.  

Following the ideas of \cite{betti2019cognitive}, we establish a link with mechanics by paralleling the weights $W$, along with the neuronal outputs $x$, to the Lagrangian coordinates of a system of particles. We define the temporal trajectories of the variables of the learning problem by laws that come from stationarity conditions of a functional, as it happens for canonical coordinates in classical mechanics. 

In particular, we indicate with $t \in [0, +\infty)$ the time variable and
we assume that the updates rules of the
parameters are obtained from the stationarity conditions of the functional
\begin{equation}
\A(x,W):=\int {1\over 2}(m_x\vert \dot x(t)\vert^2+m_W\vert \dot W(t)\vert^2)\,\varpi(t) dt+
\F(x,W),\label{cognitive-action}\end{equation}
where $m_x$ and $m_{W}$ are positive scalars (also referred to as \textit{masses}, in analogy with physics), $\varpi(t)$ is a time-dependent positive weighing term, and ${\F(x,W):=\int F(t,x,\dot x,\ddot x,W,\dot W,\ddot W)\, dt}$ (where $F$ is related to the \textit{potential} energy of the system). Notice that the kinetic energy, which appears in $\A(x,W)$,  is a sort of temporal regularization term that, once minimized, leads to develop weights that settle to constant values.
The functional in Eq.~(\ref{cognitive-action}) is subject to the previously introduced constraints
\begin{equation}G^j(t,x(t),W(t))=0,\qquad 0\le j < \nu,
\label{neuron-constraints}\end{equation}
where the map $G^j(\cdot,\cdot,\cdot)$ is taken as in
Eq.~\eqref{neuron-constraints-structure}. Notice that, in what follows, we will sometimes drop the dependency on the variables to simplify the notation.

Let $({G_{\x}\atop {G_{\W}}})$ be the $(\nu + \nu^2) \times \nu$ Jacobian matrix of the constraints
$G$ with respect to its second and third argument, respectively ($\xi$ and $M$):
\[\left({G_{\x}\atop {G_{\W}}}\right)_{ij}:=
\begin{cases}
G^j_{\xi^i} &\text{for}\quad 0 \le i< \nu,\\
\noalign{\medskip}
G^j_{\vect(M)_{i-\nu}} &\text{for}\quad \nu \le i< \nu+\nu^2,
\end{cases}
\]
where, for any matrix $M$, $\vect(M)$ is a vectorization of $M$.
Variational problems with subsidiary conditions can be tackled
using the method of Lagrange multipliers to convert the constrained problem into
an unconstrained one (see~\cite{giaquita-hildebrandt}).
In order to use this method, it is
necessary to verify the independence hypothesis between the constraints;
in this case we should check that the matrix $({G_\x\atop {G_\W}})$
is full rank.
Indeed the following proposition holds true:
\begin{proposition}
The matrix $({G_\x\atop  {G_\W}})\in {\cal M}_{(\nu^2+\nu)\times
\nu}(\R)$ is full rank.
\label{proposition}
\end{proposition}
\begin{proof}
First of all notice that if $(G_\x)_{ij}=G^j_{\x^i}$ is full rank
also $({G_\x\atop {G_\W}})$ has this property. Then, since
\[G^j_{\x^i}(\tau,\x,\W)=\begin{cases}
\delta_{ij}, & \text{if $0\le j < \omega$};\\
\delta_{ij}-\sigma'(\w_{jk} \x^k)\w_{ji} & \text{if $\omega \le j < \nu$,}
\end{cases}\]
we immediately notice that $G^i_{\x^i}=1$ and that for all $i>j$ we have
$G^i_{\x^i}=0$. This means that
\[
(G^j_{\x^i}(\tau,\x,\W))=\begin{pmatrix}
                          1&*&\cdots&*\\
                          0&1&\cdots&*\\
                          \vdots&\vdots&\ddots&\vdots\\
                          0&0&\cdots&1\end{pmatrix},\]
which is clearly full rank.
\end{proof}
This result is sufficient for the existence of Lagrange multipliers  $\lambda_j(t)$  such that
the (weak) extremals of \eqref{cognitive-action} subject to
constraints~\eqref{neuron-constraints} are extremals of the following functional:
\begin{equation}
\begin{aligned}
\A\ ^*(x,W)&=\int {1\over 2}(m_x\vert \dot x(t)\vert^2+m_W\vert \dot W(t)\vert^2)\varpi(t)\, dt\\
&-\int \lambda_j(t)G^j(t,x(t),W(t))\, dt+\F(x,W).
\end{aligned}
\end{equation}
The equations ruling the update of model parameters $(x,W)$ come from the Euler-Lagrange equations for this variational problem:
\begin{align}
  \begin{split}-m_x\varpi(t)\ddot x(t)-m_x\dot\varpi(t) \dot x(t)
  &-\lambda_j(t)G^j_\x(x(t),W(t))\\
&+L^x_F(x(t),W(t))=0;\label{x-eq}\end{split}\\
\begin{split}-m_W\varpi(t)\ddot W(t)-m_W\dot\varpi(t) \dot W(t)
&-\lambda_j(t)G^j_\W(x(t),W(t))\\
&+L^W_F(x(t),W(t))=0,\label{W-eq}\end{split}
\end{align}
where ${L_F^x=F_x-d(F_{\dot x})/dt+d^2(F_{\ddot x})/dt^2}$ and 
${L_F^W=F_W-d(F_{\dot W})/dt+d^2(F_{\ddot W})/dt^2}$
are the functional derivatives of $F$ with respect to $x$ and $W$
respectively (see \cite{giaquita-hildebrandt, courant-hilbert}).

An expression to compute the Lagrange
multipliers
is derived by differentiating two times the constraints with respect to
time and using the obtained quantities to substitute the second-order terms
in the Euler-Lagrange equations:
\begin{equation}
\begin{aligned}
\Bigl({G^i_{\xi^a}G^j_{\xi^a}\over m_x}+{G^i_{m_{ab}}G^j_{m_{ab}}
\over m_W}\Bigr)\lambda_j=&
\varpi\bigl(G^i_{\tvar\tvar}+2(G^i_{\tvar \x^a}\dot x^a\\
&+G^i_{\tvar \w_{ab}}\dot w_{ab}
+G^i_{\x^a\w_{bc}}\dot x^a\dot w_{bc})\\
&+G^i_{\x^a\x^b}\dot x^a \dot x^b
+G^i_{\w_{ab}\w_{cd}}\dot w_{ab}\dot w_{cd}\bigr)\\
&-\dot \varpi( \dot x^a G^i_{\xi^a}+\dot w_{ab} G^i_{m_{ab}})\\
&+{L^{x^a}_FG^i_{\xi^a}\over m_x}+{L^{w_{ab}}_FG^i_{m_{ab}}\over m_W},
\end{aligned}\label{mult-lin-eq-x-w}
\end{equation}
where $G^i_\tvar$, $G^i_{\tvar\tvar}$, $G^i_{\xi^a}$, $G^i_{\xi^a\xi^b}$,
$G^i_{\w_{ab}}$ and  $G^i_{\w_{ab}\w_{cd}}$ are the gradients and the
Hessians of constraint~\eqref{neuron-constraints}. 
In order to compute the multipliers $\lambda_j$, Eq. (\ref{mult-lin-eq-x-w}) must be solved. 
We conjecture that an approximate solution scheme can be implemented, by using an iterative procedure seeded with the multipliers computed at the previous time instant. However, in this paper we focus on the exact solution of Eq. (\ref{mult-lin-eq-x-w}).

\medskip
\noindent
{\bf Initial conditions. }
Now let us suppose that we want to design  an iterative update algorithm for $x$ and $W$ over time, starting from an arbitrary initial point and going on according to the learning theory so far stated. In this framework, the update algorithm comes straight from the numerical solution of the differential Eqs.~\eqref{x-eq}--\eqref{W-eq}
with Cauchy initial conditions.
Clearly, we have to choose $W(0)$ and $x(0)$ such that $g_i(0)\equiv 0$, where
we define $g_i(t):= G^i(t, x(t), W(t))$, for $i=0,\dots,\nu-1$. However, since the
constraints must hold also for all $t\ge0$, we also need
$g'_i(0)=0$. These conditions, written explicitly, mean:
\[\begin{aligned}
G^i_\tau(0,x(0),W(0))&+G^i_{\xi^a}(0,x(0),W(0))\dot x^a(0)\\
&+G^i_{m_{ab}}(0,x(0),W(0))\dot w_{ab}(0)=0.\end{aligned}\]
If the constraints do not depend explicitly on time, it is sufficient to choose $\dot x(0)=0$ and $\dot W(0)=0$, while for time dependent constraints
this condition leaves:
\[G^i_\tau(0,x(0),W(0))=0,\]
which is an additional constraint on the initial conditions $x(0)$ and $W(0)$
to be satisfied.
Therefore, one possible consistent way to impose Cauchy conditions is
\begin{equation}
\begin{aligned}
&G^i(0,x(0),W(0))=0,\quad i=0,\dots,\nu-1;\\
&G^i_\tau(0,x(0),W(0))=0,\quad i=0,\dots,\nu-1;\\
&\dot x(0)=0;\\
&\dot W(0)=0.
\end{aligned}
\label{initial-cond}
\end{equation}
Once we fix $x(0)$, $W(0)$, $\dot x(0)$ and $\dot W(0)$ as
above, higher derivatives of the variables $x$ and $W$ are determined by
the differential equations \eqref{x-eq}--\eqref{W-eq}.

\subsection{Reduction to Backpropagation}
\label{sec:bp}
Let us consider a supervised problem described by the potential (a loss
function) $V(t,x(t))$, which measures the error on the example presented at
time $t$ when the outputs of the neurons are $x(t)$;
in particular  choose
\begin{equation}
F(t,x(t),\dot x(t), \ddot x(t),W(t),\dot W(t),\ddot W(t))=-e^{\vartheta t}
V(t,x(t)) \ .
\label{lossy}
\end{equation}
where $\vartheta$ is a damping factor.
We will now show that
Eqs.~\eqref{x-eq}--\eqref{mult-lin-eq-x-w} in the limit
$m_x\to 0$, $m_W\to 0$, $m_x/m_W\to 0$ reproduce first-order
dynamics, where the updates of $W$ are performed as
prescribed by Backpropagation. In order to see this, choose $\gamma =  m_W \vartheta$ and multiply both sides of Eqs.~\eqref{x-eq}--\eqref{mult-lin-eq-x-w} by
$\exp(-\vartheta t)$, then take the limit
$m_x\to 0$, $m_W\to 0$, $m_x/m_W\to 0$. In this limit, Eq.~\eqref{W-eq} and
Eq.~\eqref{mult-lin-eq-x-w} become,  respectively:
\begin{align}
&\dot W_{ij}=-\frac1\gamma  \sigma'(w_{ik}x^k) \delta_i x^j;\label{BP-grad}\\
&G^i_{\xi^a} G^j_{\xi^a} \delta_j=-V_{x^a}G^i_{\xi^a},\label{backward}
\end{align}
where $\delta_j$ is the limit of $\exp(-\vartheta t)\lambda_j$ and
$V_x$ is the partial derivative of $V$ with respect to its second variable
($x$).
Because the matrix $G^i_{\x^a}$ is invertible Eq.~\eqref{backward} actually becomes
\begin{equation}T\delta=-V_x,\label{backward-true}\end{equation}
where $T_{ij}:= G^j_{\xi^i}$. This matrix is upper triangular, thus 
explicitly showing the backward structure of the propagation of the delta
error of the Backpropagation algorithm.
In supervised problems, $V$ depends only on the output value of the network. Hence, whenever $i$ is not an output neuron,
computing $\delta_i$ by solving Eq.~\eqref{backward-true} corresponds to:
\[\delta_i=\sigma'(w_{jk} x^k)w_{ji}\delta_j.\]
On the other hand, regarding the output units, Eq.~\eqref{backward-true} enforces the
initialization of the $\delta$s based on the value
of the supervision:
\[\delta_i=- V_{x^i}.\]
It is clear from these remarks that in Eq.~\eqref{BP-grad}
the Lagrange multiplier-related term $\delta$ plays the role of the delta error
of Backpropagation.

\section{Experimental analysis}
\label{sec:exp}
\input exp.tex

\section{Conclusions and Future Work}
\label{sec:conclusions}
We presented a theoretical study of a generic framework that describes the structure of a Neural Network by using constraints among the neural units, and that we injected  into the time domain. When data are gradually presented to the system in an online fashion, the proposed framework allowed us to devise the trajectory of the values of the weights that is optimal with respect to the Least Action Principle. We described the connection of the ideas of this paper and Backpropagation, showing that the latter can be obtained by some choices on the parameters of our model, as confirmed by an experimental analysis in which different aspects of the theory were evaluated. We plan to exploit the outcome of this work in order to handle those dynamic constraints that enforce coherence over time, such as motion coherence \cite{ijcai2019-278}, coherence on predictions over groups data points \cite{simlearn} or on space regions \cite{box,th}.

\section*{Acknowledgment}
\addcontentsline{toc}{section}{Acknowledgment}
This work was partly supported by the PRIN 2017 project RexLearn, funded by the Italian Ministry of Education, University and Research (grant no. 2017TWNMH2).

\bibliography{bib}

\end{document}

%% file: exp.tex
We investigated an easy-to-understand task in order to provide an experimental assessment of the proposed theory. 
The task consists in classifying 2-dimensional data according to the well-known non-linearly separable XOR Boolean function. 
However, this task is defined in the time domain, so that data is provided to the system in an online fashion, according to a function that basically models a trajectory in the input space, passing through the vertices of the Boolean hypercube. In particular, samples are collected along a circular trajectory, as shown in Fig. \ref{fig:trajectory}. At each time instant, a sample is provided to the system. Differently, supervision is provided to the system only when the coordinates $(e^0, e^1)$ of the input sample belong to small circular regions ($R = 0.2$) centered around the vertices of the hypercube; in this case, the target of the nearest vertex is considered.
 \begin{figure}[h!]
\centering
  \includegraphics[width=0.5\textwidth]{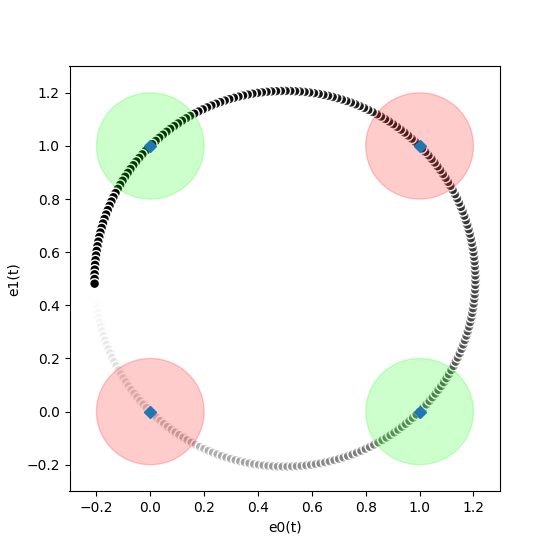}
  \caption{Trajectory along which samples are fed to the system. When we sample from within the red/green regions, a supervision is also provided. The blue diamonds are the vertices of the Boolean hypercube.}
  \label{fig:trajectory}
\end{figure}
In detail, the equations of the trajectory are
\[\begin{split}
e^0(t) &= \phi(t)  (u^0 + R \cos({\Omega t}))\\
 e^1(t) &= \phi(t)  (v^0 + R \sin({\Omega t}))\\
 \phi(t) &= \frac{1}{2}\left(1 + \tanh\Bigl(\frac{t-a}{b}\Bigr)\right)
\end{split}\]
where $u^0 = 0.5$, $v^0 = 0.5$, $a = 6$, $b=0.9$, $\Omega = 0.25$. 
The smoothing function $\phi$ is used to get a null derivative of the signals at $t=0$. 

We used a model similar to the one that we introduced as running example in Section \ref{sec:constraints}, i.e. a one-hidden-layer feed-forward neural network with hyperbolic tangent as activation function $\sigma$ and biases, as shown in Fig. \ref{fig:network}. We used a potential (loss) term that is structured as Eq. \eqref{lossy} and the loss function $V$ is the classic Mean Squared Error (MSE).
Model accuracy and the loss value are computed either on the Boolean hypercube corners or on a set of points sampled from uniform distributions over the supervision regions. In particular, we used the notation \textit{Acc}, \textit{Loss} to indicate the former case, while \textit{Acc2}, \textit{Loss2} are about the latter. Each experiment consists of $10$ runs  with different random initializations of $W \sim U(-\sqrt{2}, \sqrt{2}\,)$, being $U$ the uniform distribution. Results are averaged over the runs ($\pm$ standard deviation).
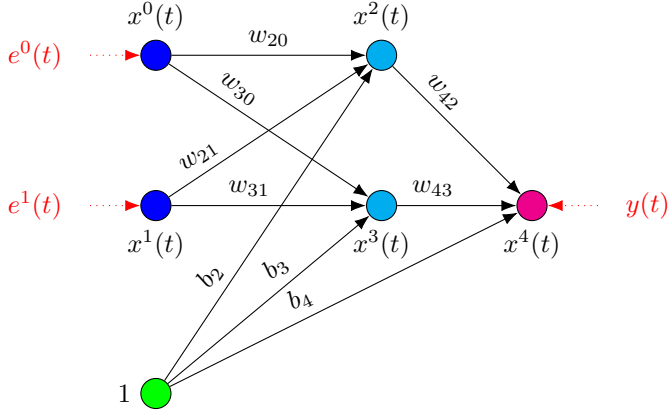
\begin{figure}[h!]
\def\a{-{Latex[length=2mm]}}
\begin{tikzpicture}
[   cnode/.style={draw=black,fill=#1,minimum width=4mm,circle},
]
    \node[cnode=magenta,label=270:$x^4(t)$] (s) at (5,-2) {};
    \node[label=0:$\color{red}y(t)$] (y) at (6,-2) {};
    \node[label=180:$\color{red}e^0(t)$] (e0) at (-1,-0) {};
    \node[label=180:$\color{red}e^1(t)$] (e1) at (-1,-2) {};
   
    \node[cnode=blue,label=90:$x^0(t)$] (x-0) at (0,{-0}) {};
    \node[cnode=blue,label=270:$x^1(t)$] (x-1) at (0,{-2}) {};
    \node[cnode=green,label=180:$1$] (b) at (0,{-4.5}) {};
   
    \draw[\a,dotted,color=red] (e0) -- (x-0);
    \draw[\a,dotted,color=red] (e1) -- (x-1);
   
    \node[cnode=cyan,label=90:$x^2(t)$] (x-2) at (3,{-0}) {};
    \node[cnode=cyan,label=270:$x^3(t)$] (x-3) at (3,{-2}) {};
   
    \draw[\a] (x-3) -- node[above,sloped,pos=0.3] {$w_{43}$} (s);
    \draw[\a] (x-2) -- node[above,sloped,pos=0.3] {$w_{42}$} (s);
   
    \draw[\a] (x-0) -- node[above,sloped] {$w_{20}$} (x-2);
    \draw[\a] (x-0) -- node[above,sloped,pos=0.3] {$w_{30}$} (x-3);
    \draw[\a] (b) -- node[above,sloped,pos=0.3] {$b_2$} (x-2);
    \draw[\a] (b) -- node[above,sloped,pos=0.6] {$b_3$} (x-3);
    \draw[\a] (b) -- node[above,sloped,pos=0.4] {$b_4$} (s);
    \draw[\a] (x-1) -- node[above,sloped,pos=0.2] {$w_{21}$} (x-2);
    \draw[\a] (x-1) -- node[above,sloped,pos=0.4] {$w_{31}$} (x-3);
   
    \draw[\a,dotted,color=red] (y) -- (s);
\end{tikzpicture}
\caption{Neural architecture: $e^0(t)$ and $e^1(t)$ are the input signals, while $y(t)$ is the supervision signal.}
\label{fig:network}
\end{figure}


We compare the proposed approach with a \textit{baseline} system that exploits stochastic gradient descent to optimize the loss function $V$ only, operating in the same online learning setting described above. 
The considered task generates samples that have a very high temporal correlation, making the task not trivial. 
Table~\ref{tab:baseline} reports the results of the baseline model, as a function of the choice of the learning rate $\eta$. It is interesting to notice that such model fails to correctly classify the data in most of the cases, being it extremely sensitive to the choice of $\eta$, as expected.

The proposed learning approach involves solving differential equations of a \textit{second-order model} (see Section \ref{sec:least}, Eq. \ref{x-eq}, Eq. \ref{W-eq}). We used the LSODA solver, available in the ODEPACK Fortran library. The algorithm, originally proposed in \cite{lsoda}, implements Adams and BDF methods with automatic stiffness detection and switching. We set $\Delta t = 0.1$ and error control performed by the solver is determined by $rtol = atol = 1.49 \cdot 10^{-8}$. 
We also consider a \textit{first-order model} that is implemented by assuming to have reached the limit conditions of (Section \ref{sec:bp}). 
On the other hand, for 1st-order equations (Section \ref{sec:bp}) we used the Euler method, in order to reproduce the small weight updates typical of gradient descent methods. In all the following experiments, we ensured that the learning rate of the baseline model, $\eta$, is chosen coherently with the scaling factors of Eq. \ref{BP-grad} and the step-size $\Delta t$. In particular, we set
$\eta= \Delta t/\gamma = \Delta t/(m_W \vartheta)$.

%

%


\begin{table}[htbp]
\caption{Performance of the baseline model as a function of the learning rate $\eta$.}
  \label{tab:baseline}
\begin{center}  
  \begin{tabular}{|c|c|c|c|c|}
    \hline
    $\eta$ &
     Acc & Loss & Acc2 & Loss2 \\
\hline
0.001 & 0.650 {\tiny $\pm$ 0.122} & 0.906 {\tiny $\pm$ 0.133} & 0.651 {\tiny $\pm$ 0.117} & 0.911 {\tiny $\pm$ 0.132} \\                                0.0025 & 0.700 {\tiny $\pm$ 0.150} & 0.759 {\tiny $\pm$ 0.209} & 0.697 {\tiny $\pm$ 0.147} & 0.767 {\tiny $\pm$ 0.205} \\                               0.01 & 0.825 {\tiny $\pm$ 0.115} & 0.475 {\tiny $\pm$ 0.283} & 0.823 {\tiny $\pm$ 0.116} & 0.491 {\tiny $\pm$ 0.278} \\                                 0.03 & 0.875 {\tiny $\pm$ 0.125} & 0.398 {\tiny $\pm$ 0.388} & 0.873 {\tiny $\pm$ 0.128} & 0.410 {\tiny $\pm$ 0.392} \\                                 0.1 & 0.900 {\tiny $\pm$ 0.166} & 0.305 {\tiny $\pm$ 0.472} & 0.902 {\tiny $\pm$ 0.161} & 0.313 {\tiny $\pm$ 0.476} \\                                  0.125 & 0.900 {\tiny $\pm$ 0.166} & 0.315 {\tiny $\pm$ 0.492} & 0.901 {\tiny $\pm$ 0.163} & 0.321 {\tiny $\pm$ 0.496} \\                                0.3 & 0.900 {\tiny $\pm$ 0.166} & 0.349 {\tiny $\pm$ 0.560} & 0.900 {\tiny $\pm$ 0.166} & 0.353 {\tiny $\pm$ 0.562} \\                                  0.5 & 0.900 {\tiny $\pm$ 0.166} & 0.369 {\tiny $\pm$ 0.601} & 0.900 {\tiny $\pm$ 0.166} & 0.371 {\tiny $\pm$ 0.602} \\                                  0.8 & 0.950 {\tiny $\pm$ 0.100} & 0.197 {\tiny $\pm$ 0.393} & 0.950 {\tiny $\pm$ 0.101} & 0.199 {\tiny $\pm$ 0.397} \\                                  1 & 0.975 {\tiny $\pm$ 0.075} & 0.096 {\tiny $\pm$ 0.288} & 0.974 {\tiny $\pm$ 0.076} & 0.102 {\tiny $\pm$ 0.293} \\                                    2 & 0.550 {\tiny $\pm$ 0.100} & 1.800 {\tiny $\pm$ 0.400} & 0.550 {\tiny $\pm$ 0.099} & 1.802 {\tiny $\pm$ 0.395} \\                                    3 & 0.525 {\tiny $\pm$ 0.075} & 1.900 {\tiny $\pm$ 0.300} & 0.524 {\tiny $\pm$ 0.071} & 1.904 {\tiny $\pm$ 0.287} \\
\hline

  \end{tabular}
\end{center}
\end{table}
\subsection{Comparison between First-Order Model and Baseline}\label{sec:1st-vs-bs}
In Section \ref{sec:bp} we discussed the relationships between a special instance of our model, that is the first-order case, and Backpropagation.

We compared the trajectories of the weights and biases comparing such model with the baseline case, setting ${\eta = \Delta t/\gamma = 0.3}$ using the previously described criterion, and reporting them in Fig. \ref{fig:1st-vs-bs}. 
For each parameter (i.e., weights and biases) we have two trajectories (same color, different style), one from our model, one from the baseline system. It is interesting to see that the two dynamics converge to almost identical values (differences are due to round-off errors), confirming the correctness of the results of Section \ref{sec:bp}.
\begin{figure}[h]
\centering
  \includegraphics[width=\linewidth,trim={10 0 0 0},clip]{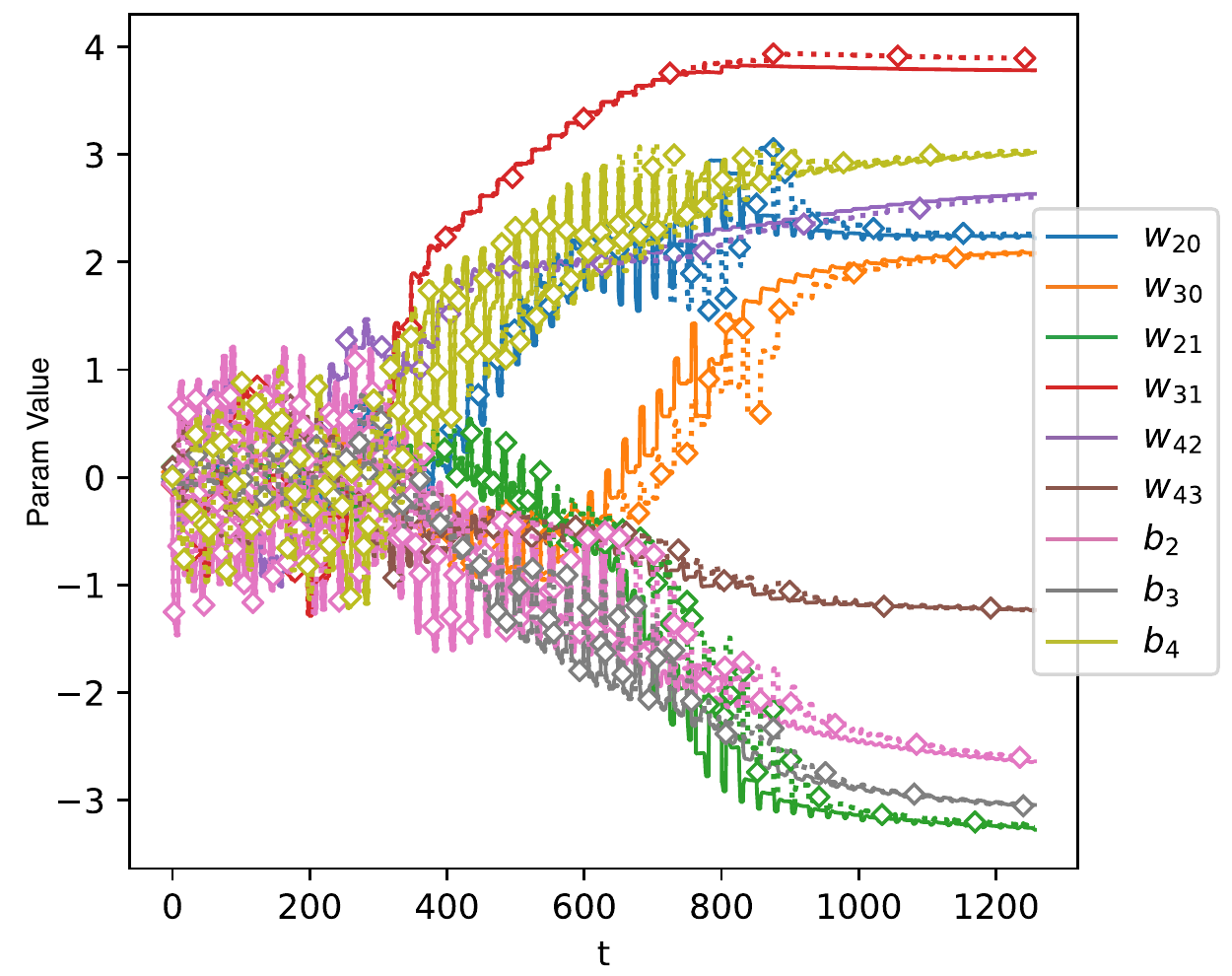}
  \caption{Comparison between the dynamics of the weights and biases of the first-order model (lines) and the baseline system (dotted lines with markers). The value of the weights (or biases) is plotted against time. Trajectories with the same colors are about the same parameter. The two dynamics converge to almost identical values, indistinguishable in the first time instants (differences for larger $t$ are due to numerical errors).}
  \label{fig:1st-vs-bs}
\end{figure}

\subsection{Comparison between Second and First-Order Model}\label{sec:1st-vs-2nd}
Section \ref{sec:bp} suggests that the second-order model can be moved very close to a first-order model when using small masses and strong damping.
We compared the weights (and biases) dynamics of the second-order model, under the aforementioned conditions, with the first-order model.
In particular, we set ${\vartheta = 333.3}, {m_x =1.0 \cdot 10^{-5}}, {m_W = 1.0 \cdot 10^{-3}}$ and $\gamma = \vartheta \cdot m_W$.
Fig. \ref{fig:1st-vs-2nd} (that follows the same organization of Fig. \ref{fig:1st-vs-bs}) shows that the two algorithms exhibit almost identical behaviour.
\begin{figure}[h]
\centering
  \includegraphics[width=0.85\linewidth,trim={10 0 0 0},clip]{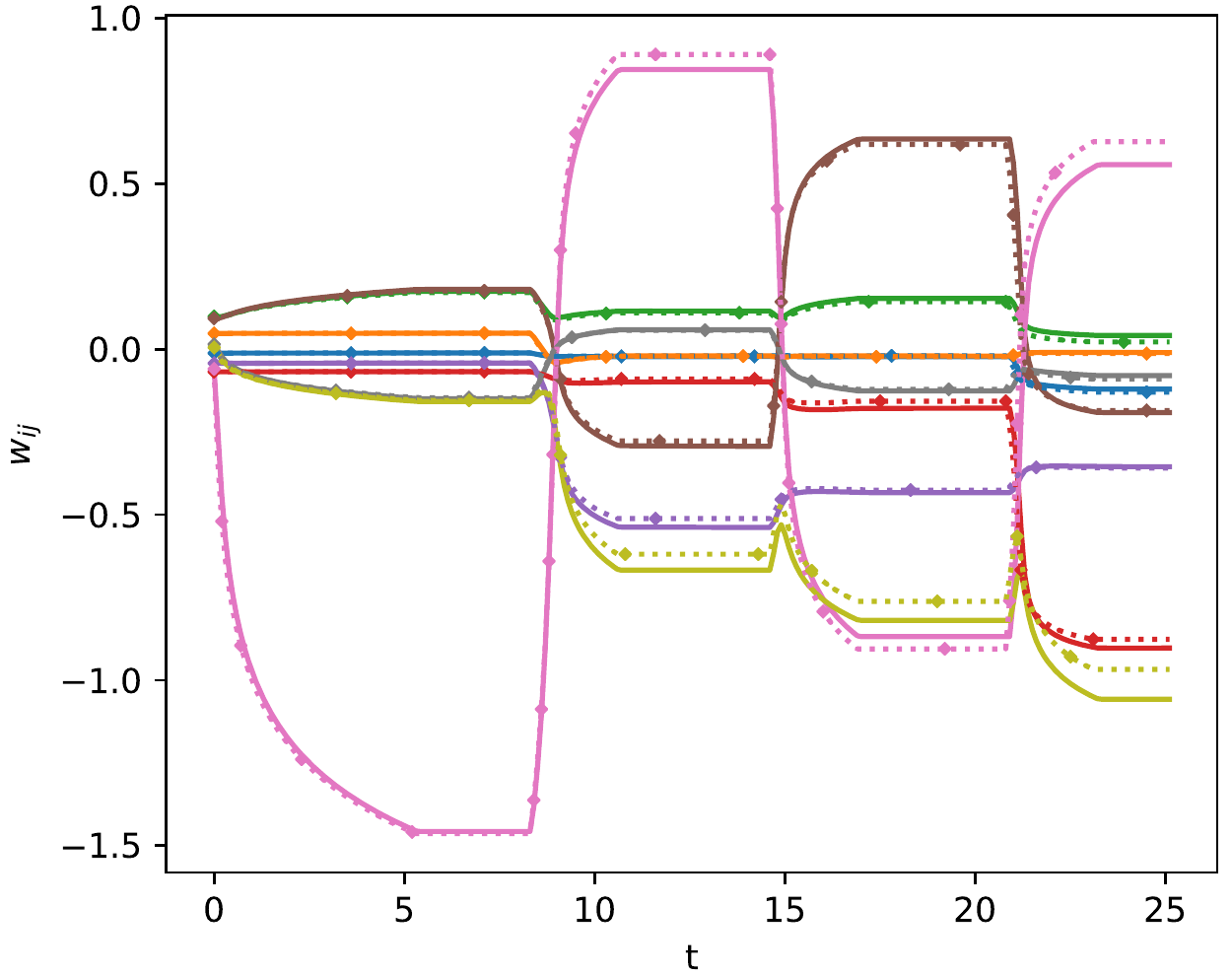}
  \caption{Comparison between the dynamics of the weights and biases of the second-order model (lines) and the first-order one (dotted lines with markers). The value of the weights (or biases) is plotted against time. Trajectories with the same colors are about the same parameter (refer to legend of Fig. \ref{fig:1st-vs-bs}). The two algorithms exhibit almost identical behavior.}
  \label{fig:1st-vs-2nd}
\end{figure}

Notice that to obtain equivalent differential problem formulations it is required to have coherent initial conditions.
In the case of the second-order model, the initial value of derivatives of variables is arbitrary, while in the first-order model it is given by the differential equation itself. An easy way to ensure consistency between the formulations is to multiply, in both the models, the potential term $V$, with a weighing function $(1-e^{\vartheta t})$ that, under the assumption of strong damping, gets to $1$ very quickly. Moreover, with this choice we obtain $\dot{W}(0) = 0$, which implies coherence of initial conditions of the two problems and easier convergence of the solution of the second-order problem to the first-order one. The latter is essentially a continuous-time online learning gradient descent algorithm (as the baseline system), with the only difference being the time-dependent factor $(1-e^{\vartheta t})$. When altering the potential term as suggested, we get a stronger convergence of the derivatives at the left boundary. On the other hand, the overall slight difference between the weights trajectories is motivated by the non-zero inertial properties of the second-order formulation.
\begin{figure}[h]
\centering     
\subfigure[w/o $(1-e^{-\vartheta t})$]{\includegraphics[width=.2\textwidth]{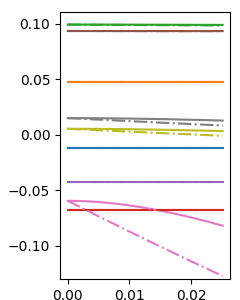}}
\subfigure[with $(1-e^{-\vartheta t})$]{\includegraphics[width=.2\textwidth]{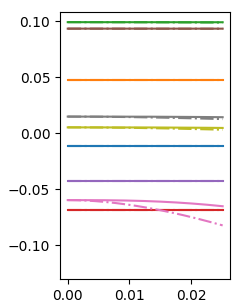}}

\caption{\label{fig:1st-vs-2nd-convergence} Comparison between the dynamics of the weights and biases of the second-order model (lines) and the first-order one (dotted lines with markers). The value of the weights (or biases) is plotted against time. Trajectories with the same colors are about the same parameter (refer to legend of Fig. \ref{fig:1st-vs-bs}). (a) Without altering the potential term with the weighing function $(1-e^{-\vartheta t})$, (b) using the altered potential term.
Notice that (b) shows a stronger coherence close to the boundary $t=0$.}
\end{figure}

\subsection{Comparison between Second-Order Model and Baseline}
We performed an extended comparison between the proposed second-order model and the baseline system. 
We considered different sets of values for the key parameters of our theory, that are the masses $m_x$, $m_W$ and the damping term $\vartheta$. The learning rate $\eta$ of the baseline model is computed as a function of such parameters: $\eta= \Delta t/\gamma = \Delta t/(m_W \vartheta)$ (further details in the introduction of Section \ref{sec:exp}).

Results are reported in Tab. \ref{tab:baseline-vs-stlp}.
When small masses and large damping coefficients are used, the compared systems perform in a very similar way (rows $1,2,4$). Otherwise, the behaviours depart one from each other, since the dynamics of the second-order model are clearly more structured that in the vanilla baseline.
We also measured the norms of the constraining functions involving the
neural units at the last time instant $T$
\[\vert g(T)\vert=\biggl(
\sum_{j=0}^{\nu-1}\Bigl(G^j(T,x(T), W(T))\Bigr)^2\biggr)^{1/2},\]
and we report them in the fifth column of Tab. \ref{tab:baseline-vs-stlp}. Overall, constraints are fulfilled in all the evaluated settings, confirming the soundness of the proposed update scheme for the Lagrange multipliers of Eq. (\ref{mult-lin-eq-x-w}).
\begin{table*}[h]
\caption{Comparison of the second-order model with the baseline system, with different configurations of the model parameters. Rows 1,2,4 are associated with small masses and strong damping. Baseline results are taken form Tab. \ref{tab:baseline}.}
\label{tab:baseline-vs-stlp}
\begin{center}
  \begin{tabular}{|c|c|c|c||c|c|c|c|c||c|c|c|}
    \hline
      \multicolumn{4}{|c||}{} & 
      \multicolumn{5}{c||}{\textsc{Second-Order Model}} & \multicolumn{2}{c|}{\textsc{Baseline}} \\
       \cline{1-11}
      {$\eta$} & $m_x$ & $m_W$ & $\vartheta$  & $\vert g(T)\vert$  & Acc  & Loss & Acc2  & Loss2 &  Acc & Loss \\
      \hline
    0.3 &  $1.0 \cdot 10^{-4}$ & $1.0 \cdot 10^{-2}$ & 33.3 & 0.030 {\tiny $\pm$ 0.007} & 0.900  {\tiny $\pm$  0.166} & 0.348 {\tiny $\pm$  0.557}  & 0.899 {\tiny $\pm$ 0.167} & 0.354 {\tiny $\pm$ 0.557} & 0.900 {\tiny $\pm$ 0.166} & 0.349 {\tiny $\pm$ 0.560} \\ 
    0.03 &  $1.0 \cdot 10^{-4}$ & $1.0 \cdot 10^{-2}$ & 333 & 0.034 {\tiny $\pm$ 0.003} & 0.850 {\tiny $\pm$  0.166}                                                                                                                          & 0.436 {\tiny $\pm$  0.436} & 0.848 {\tiny $\pm$ 0.167} & 0.450 {\tiny $\pm$ 0.436} & 0.875 {\tiny $\pm$ 0.125} & 0.398 {\tiny $\pm$ 0.388}\\
    0.3 &  $1.0 \cdot 10^{-4}$ & $5.0 \cdot 10^{-1}$ & 0.667 & 0.030 {\tiny $\pm$ 0.007} & 0.625  {\tiny $\pm$  0.125} & 1.499  {\tiny $\pm$  0.499}  & 0.624  {\tiny $\pm$  0.124} & 1.504  {\tiny $\pm$ 0.494} & 0.900 {\tiny $\pm$ 0.166} & 0.349 {\tiny $\pm$ 0.560} \\
    0.125 &  $1.0 \cdot 10^{-4}$ & $1.0 \cdot 10^{-2}$ & 80.0 & 0.031  {\tiny $\pm$  0.002} &  0.850  {\tiny $\pm$  0.200}  &  0.466  {\tiny $\pm$  0.591}  & 0.851  {\tiny $\pm$  0.199} &  0.475  {\tiny $\pm$  0.590}  & 0.900 {\tiny $\pm$ 0.166} & 0.315 {\tiny $\pm$ 0.492}\\
    0.8 &  $1.0 \cdot 10^{-4}$ & $1.0 \cdot 10^{-2}$ & 12.5  & 0.031  {\tiny $\pm$  0.004}  &  0.950  {\tiny $\pm$  0.100}  &  0.197  {\tiny $\pm$  0.393}  &  0.950  {\tiny $\pm$  0.100}  &  0.200  {\tiny $\pm$  0.394} & 0.950 {\tiny $\pm$ 0.100} & 0.197 {\tiny $\pm$ 0.393} \\
    0.0025 & 1.0 & 5.0 & 8.0 & 0.030  {\tiny $\pm$  0.001}  &  0.450  {\tiny $\pm$  0.187}  &  1.156  {\tiny $\pm$  0.231}  &  0.467  {\tiny $\pm$  0.179}  &  1.150  {\tiny $\pm$  0.227} & 0.700 {\tiny $\pm$ 0.150} & 0.759 {\tiny $\pm$ 0.209}\\
    \hline
  \end{tabular}
\end{center}
\end{table*}
Interestingly, the differences between $Acc$ and $Acc2$ (and also $Loss$ and $Loss2$) are minimal, suggesting that the classifier is keeping some margin among the vertices of the Boolean hypercube. We further investigated this aspect by visualizing the decision boundaries developed by the model, reported in Fig. \ref{dec2_row5} (related to the setting of row $5$ in Tab. \ref{tab:baseline-vs-stlp}), which confirms the previous considerations. The whole supervision-related areas are correctly classified, even if the sampling trajectory only intercepts a small portion of such areas, i.e. the arcs of circumference contained within the supervision regions. For completeness, we also report the same picture in the case of the baseline model, Fig. \ref{decbase}, that leads to a similar result in terms of margin (we randomly selected one of the $10$ runs).
\begin{figure}[h]
\centering
  \includegraphics[width=0.8\linewidth]{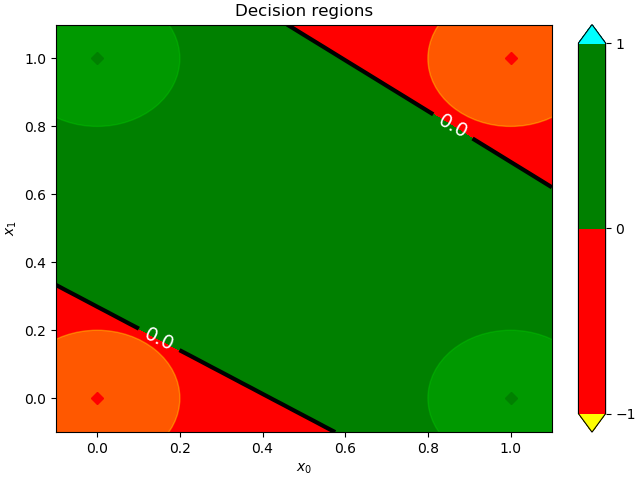}
  \caption{Second-order model. Decision regions of the network resulting from one of the runs using parameters in row 5 of Tab. \ref{tab:baseline-vs-stlp}  (random initialization).}
  \label{dec2_row5}
\end{figure}

\begin{figure}[h]
\centering
  \includegraphics[width=0.8\linewidth]{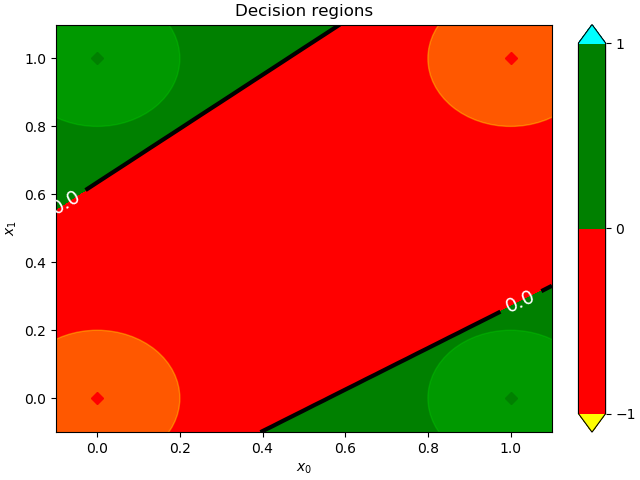}
  \caption{Baseline model. Decision regions of the network resulting from one of the runs using parameters in row 5 of Tab. \ref{tab:baseline-vs-stlp} (random initialization).}
  \label{decbase}
\end{figure}

We analyzed the weight dynamics in three configurations that we selected in order to move from (a) a uniformly weighed case ($m_x=m_W=\vartheta=1$), (b) an intermediate setting ($m_x=0.001$, $m_W=0.1$, $\vartheta=5$) and (c) the already mentioned small-masses-large-damping case ($m_x=0.0001$, $m_W=0.001$, $\vartheta=33$). Fig. \ref{fig:comp1}, top-row, show the dynamics of the second-order model, while Fig. \ref{fig:comp1}, bottom-row, is about the baseline model (having only parameter $\eta$, chosen as previously mentioned in order to produce comparable behavior).
Oscillations are significant (case (a)), unless damping is remarkably high (case (c)).
As expected, when inertial properties of the variables are not negligible, the dynamics are smoother and less peaky with respect to signal variations (slower behavior). 
\begin{figure*}[h]
\centering     
\rotatebox{90}{\hskip 1.3cm Second-Order Model}\hskip 2mm{\includegraphics[width=.31\textwidth]{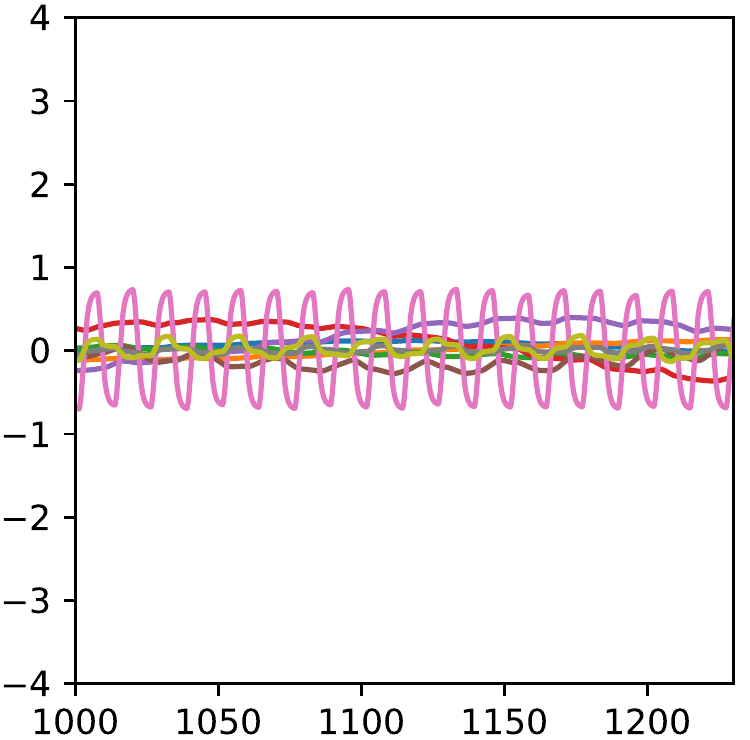}}
{\includegraphics[width=.31\textwidth]{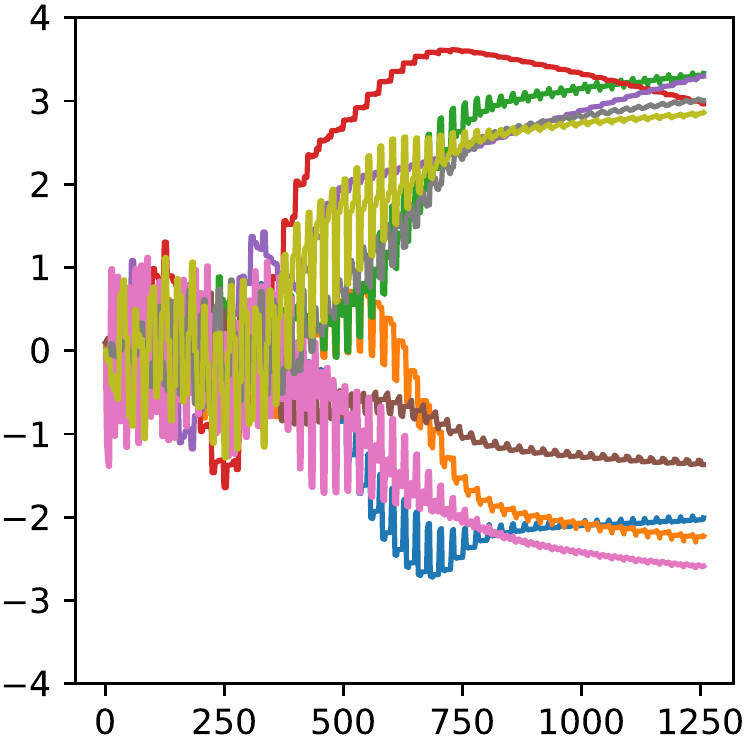}}
{\includegraphics[width=.31\textwidth]{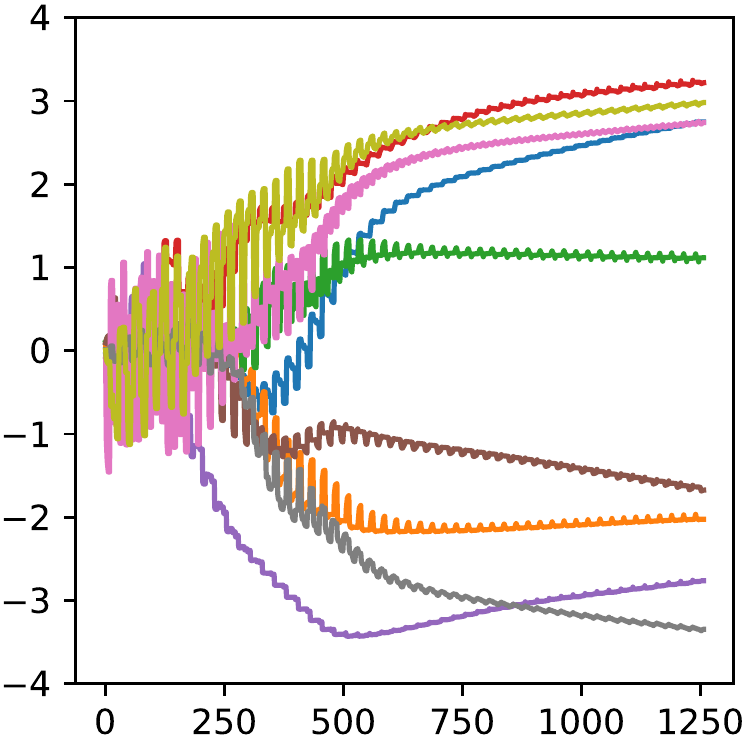}}\\
\rotatebox{90}{\hskip 1.7cm Baseline Model}\hskip 2mm\subfigure[$m_x=m_W=\vartheta=1$]{\includegraphics[width=.31\textwidth]{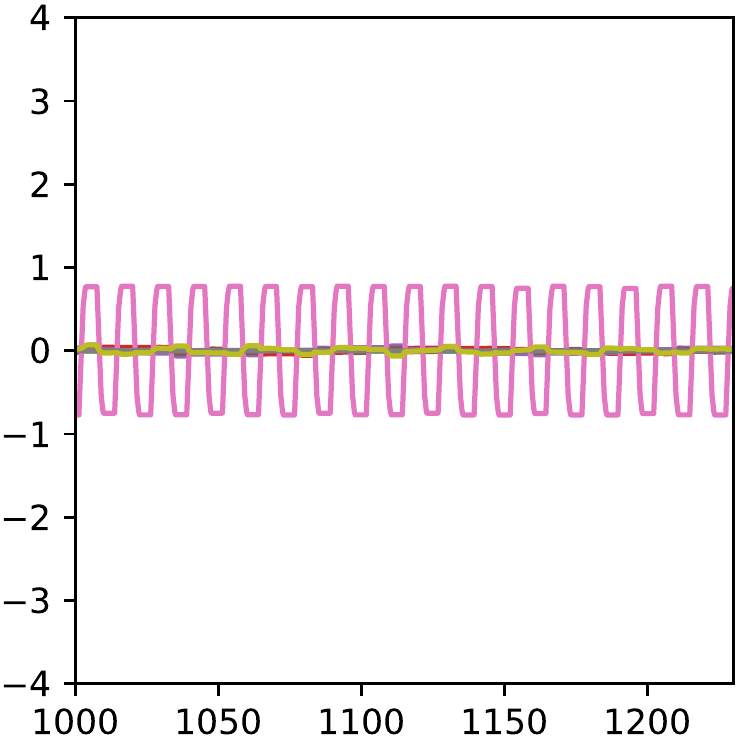}}
\subfigure[$m_x=0.001$, $m_W=0.1$, $\vartheta=5$]{\includegraphics[width=.31\textwidth]{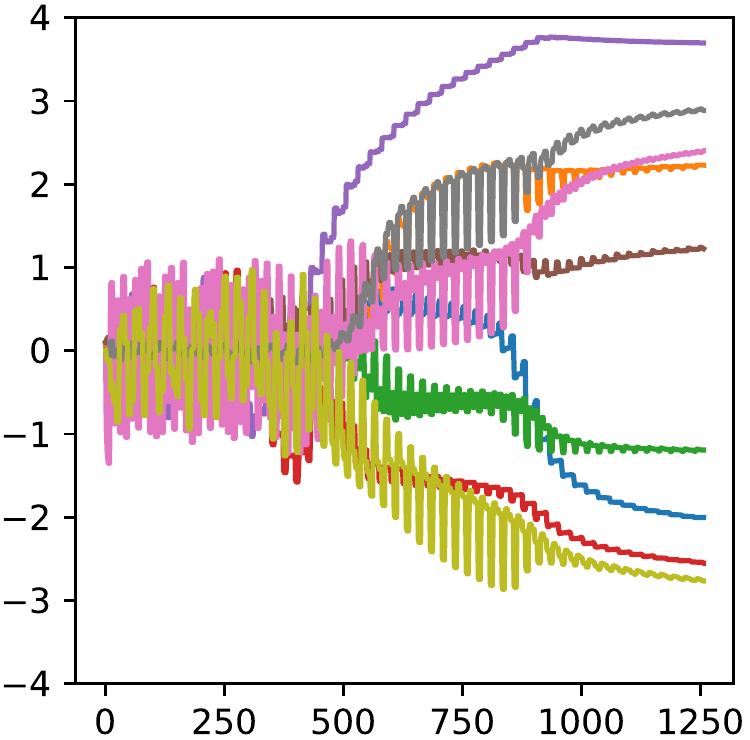}}
\subfigure[$m_x=0.0001$, $m_W=0.001$, $\vartheta=33$]{\includegraphics[width=.31\textwidth]{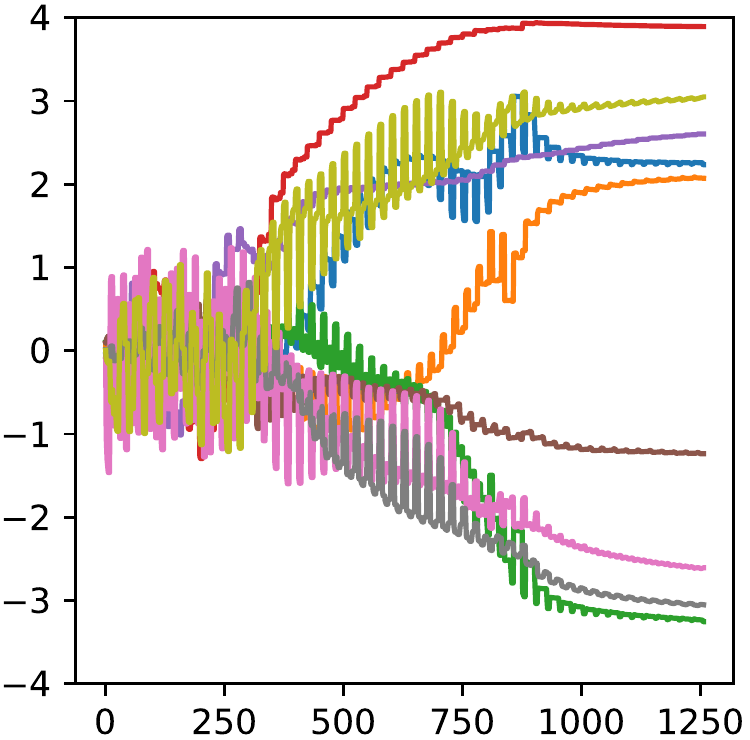}}
\caption{Evolution of the weights (and biases) for different configurations of the parameters (one for each column, (a),(b),(c)). Each curve shows the value (ordinate) of a network parameter versus time (abscissa). Refer to legend of Fig. \ref{fig:1st-vs-bs}. Top row: second-order model. Bottom row: baseline model (the learning rate $\eta$ is a function of $m_W, \vartheta$ - see the text of the paper for further details).}
\label{fig:comp1}
\end{figure*}